%% file: SIIMS.tex
\documentclass{siamart0516}

\usepackage{appendix}
\usepackage{enumitem}

\DeclareMathOperator*{\argmin}{argmin}

\usepackage{subfigure}
\input{SIIMS_shared.tex}

\newcommand{\RR}{{\mathbb R}}
\renewcommand{\Re}{{\mathbb R}}

\newcommand{\w}{{\mathbf w}}
\newcommand{\eps}{\epsilon}
\renewcommand{\u}{{\mathbf u}}

\renewcommand{\v}{{\mathbf v}}
\newcommand{\F}{{\mathbf F}}
\newcommand{\G}{{\mathbf G}}
\newcommand{\T}{{\mathbf T}}
\renewcommand{\H}{{\mathbf H}}

\newcommand{\I}{{\mathbf I}}

\newcommand{\0}{{\bf 0}}

\newcommand{\bu}{\bar{\mathbf u}}
\newcommand{\bv}{\bar{\mathbf v}}

\newcommand{\xhat}{\widehat{x}}

\begin{document}

%%%%%%%%% TITLE

%%% Title and other info are in SIIMS_shared.tex %%%

\maketitle
%\thispagestyle{empty}

%%%%%%%%% ABSTRACT
\begin{abstract}
Regularized inversion methods for image reconstruction are used widely
due to their tractability and their ability to combine complex physical sensor models with useful regularity criteria.  
Such methods motivated the recently developed Plug-and-Play prior method, 
which provides a framework to use advanced denoising algorithms as regularizers in inversion.
However, the need to formulate regularized inversion as the solution to an optimization problem limits the expressiveness of possible regularity conditions and physical sensor models. 

In this paper, we introduce the idea of Consensus Equilibrium (CE), which generalizes regularized inversion to include a much wider variety of both forward (or data fidelity) components and prior (or regularity) components without the need for either to be expressed using a cost function.
Consensus equilibrium is based on the solution of a set of equilibrium equations that balance data fit and regularity.
In this framework, the problem of MAP estimation in regularized inversion is replaced by the problem of solving these equilibrium equations, which can be approached in multiple ways.  

The key contribution of CE is to provide a novel framework for fusing multiple heterogeneous models of physical sensors or models learned from data.
We describe the derivation of the CE equations and prove that the solution of the CE equations generalizes the standard MAP estimate under appropriate circumstances.  

We also discuss  algorithms for solving the CE equations, including a version of the Douglas-Rachford (DR)/ADMM algorithm with a novel form of preconditioning and Newton's method, both standard form and a Jacobian-free form using Krylov subspaces. We give several examples to illustrate the idea of consensus equilibrium and the convergence properties of these algorithms and demonstrate this method on some toy problems and on a denoising example in which we use an array of convolutional neural network denoisers, none of which is tuned to match the noise level in a noisy image but which in consensus can achieve a better result than any of them individually.  
\end{abstract}

% REQUIRED

% REQUIRED

%%%%%%%%% BODY TEXT
\section{Introduction}

Over the past 30 years, statistical inversion has evolved from an interesting theoretical idea to a proven practical approach. 
Most statistical inversion methods are based on the maximum a posteriori (MAP) estimate, 
or more generally regularized inversion, using a Bayesian framework, since this approach balances computational complexity with achievable image quality. 
In its simplest form, regularized inversion is based on the solution of the optimization problem 
\begin{equation}
x^* = \argmin_{x} \left\{ f(x) + h(x) \right\},
\label{eq:RegularizedInverse}
\end{equation}
where $f$ is the data fidelity function and $h$ is the regularizing function.
In the special case of MAP estimation, $f$ represents the forward model and $h$ represents the prior model, given by 
$$
f(x) = - \log p_\text{forward}(y|x), \ \ \ \ \ h(x) = - \log p_\text{prior}(x),
$$ 
where $y$ is the data and $x$ is the unknown to be recovered.  
The solution of equation~(\ref{eq:RegularizedInverse}) balances the goals of fitting the data while also regularizing this fit according to the prior. 

In more general settings, for example with multiple data terms from multi-modal data collection, a cost function can be decomposed as a sum of auxiliary (usually convex) functions:
$$ \text{minimize } f(x) = \sum_{i=1}^N f_i(x), $$
with variable $x \in \RR^n$ and $f_i : \RR^n \to \RR \cup \{+\infty\}$.  In consensus optimization, the minimization of the original cost function is reformulated as the minimization of the sum of the auxiliary functions, each a function of a separate variable, with the constraint that the separate variables must share a common value:
\begin{equation} \label{eq:split-no-w}
 \text{minimize } \sum_{i=1}^N f_i(x_i) \text{ subject to } x_i = x, \ i=1, \ldots, N, 
 \end{equation}
with variables $x \in \RR^n$, $x_i \in \RR^n$, $i = 1, \ldots, N$. 
This reformulation allows for the application of the Alternating Direction Method of Multipliers (ADMM) or other efficient minimization methods and applies to the original problem in \eqref{eq:RegularizedInverse} as well as many other problems.  An account of this approach with many variations and examples can be found in \cite{boyd2011distributed}.

While regularized inversion and optimization problems more generally benefit from extensive theoretical results and powerful algorithms, they are also expressively limited.  For example, many of the best denoising algorithms cannot 
be put into the form of a simple optimization \cite{buades2005review, DabovBM3D07}.  Likewise, the behavior of denoising neural networks cannot generally be captured via optimization.  
These successful approaches to inverse problems lie outside the realm of optimization problems and give rise to the motivating question for this paper:

\bigskip \noindent
{\bf Question:}  How can we generalize the consensus optimization framework in \eqref{eq:split-no-w} to encompass models and operators that are not associated with an optimization problem, and how can we find solutions efficiently?  

\bigskip
There is a vast and quickly-growing literature on methods and results for convex and consensus optimization.  Seminal work in this area includes the work of Lions and Mercier \cite{Lions1979}, as well as the PhD thesis of Eckstein \cite{ecksteinThesis} and the work of Eckstein and Bertsekas \cite{eckstein1992douglas}.   We do not provide a complete survey of this literature since our focus is on a framework beyond optimization, but some starting points for this area are \cite{Bauschke2011,boyd2011distributed,BoydVandenberghe2009}.  

As for approaches to fuse a data fidelity model with a denoiser that is not based on an optimization problem, the first attempt to our knowledge is \cite{venkatakrishnan2013school}.  The goal of this approach, called the Plug-and-Play prior method, is to replace the prior model in the Bayesian formulation with a denoising operator.  This is done by taking the ADMM algorithm, which is often used to find solutions for consensus optimization problems, and replacing one of the optimization steps (proximal maps) of this algorithm with the output of a denoiser.  Recently, a number of authors have built on the Plug-and-Play method as a way 
to construct implicit prior models through the use 
of denoising operators \cite{rond2015poisson, sreehari2016TCI, Sreehari2016, venkatakrishnan2013model}.
In \cite{sreehari2016TCI}, conditions are given on the denoising operator that will ensure it is a proximal mapping,
so that the MAP estimate exists and the ADMM algorithm converges.
However, these conditions impose relatively strong symmetry conditions on the denoising operator that may not occur in practice. For applications where fixed point convergence is sufficient, it is possible to relax the conditions on the denoising operator by iteratively controlling the step size in the proximal map for the forward model and the noise level for the denoiser \cite{Chan_Wang_Elgendy_2017}.

The paper \cite{Romano2017} provides a different approach to building on the idea of Plug-and-Play.
That paper uses the classical forward model plus prior model in the framework of optimization, but constructs a prior term directly from the denoising engine; this is called Regularization by Denoising (RED).  
For a denoiser $x \mapsto H(x)$, the prior term is given by $\lambda x^T (x - H(x))$.
This approach is formulated as an optimization problem associated with any denoiser, 
but in the case that the denoiser itself is obtained from a prior, the RED prior is different from the denoiser prior.
Other approaches that build on Plug-and-Play include \cite{Ono-2017}, which uses primal-dual splitting in place of an ADMM approach, and \cite{Kamilov-2017}, which uses FISTA in a Plug-and-Play framework to address a nonlinear inverse scattering problem.  

In this paper, we introduce Consensus Equilibrium (CE) as an optimization-free generalization of regularized inversion and consensus optimization that can be used to fuse multiple sources of information implemented as maps such as denoisers, deblurring maps, data fidelity maps, proximal maps, etc. 
We show that CE generalizes consensus optimization problems in the sense that if the defining maps are all proximal maps associated with convex functions, then any CE solution is also a solution to the corresponding consensus optimization problem.
However, the consensus equilibrium can still exist in the more general case when the defining maps are not proximal maps; in this case, there is no underlying optimization. 
In the case of a single data fidelity term and a single denoiser, the solution has the interpretation of achieving the best denoised inverse of the data. 
That is, the proximal map associated with the forward model pulls the current point towards a more accurate fit to data, while the denoising operator pulls the current point towards a ``less noisy'' image. 
We illustrate this in a toy example in two dimensions:  the consensus equilibrium is given by a balance between two competing forces. 

In addition to introducing the CE equations, we discuss ways to solve them and give several examples.  We describe a version of the Douglas-Rachford (DR)/ADMM algorithm with a novel form of anisotropic preconditioning.  We also apply Newton's method, both in standard form and in a Jacobian-free form using Krylov subspaces.  

In the experimental results section, we give several examples to illustrate the idea of consensus equilibrium and the convergence properties of these algorithms.
We first demonstrate the proposed algorithms on some toy problems in order to illustrate properties of the method.
We next use the consensus equilibrium framework to solve an image denosing problem using an array of convolutional neural network (CNN) denoisers, none of which is tuned to match the noise level in a noisy image.
Our results demonstrate that that the consensus equilibrium result is better than any of the individually applied CNN denoisers.

\section{Consensus Equilibrium: Optimization and Beyond}

In this section we formulate the consensus equilibrium equations, show that they encompass a form of consensus optimization in the case of proximal maps, and describe the ways that CE goes beyond the optimization framework.  

\subsection{Consensus Equilibrium for Proximal Maps}

We begin with a slight generalization of \eqref{eq:split-no-w}:  
\begin{equation} \label{eq:split}
 \text{minimize } \sum_{i=1}^N \mu_i f_i(x_i) \text{ subject to } x_i = x, \ i=1, \ldots, N, 
 \end{equation}
with variables $x \in \RR^n$, $x_i \in \RR^n$, $i = 1, \ldots, N$, and weights $\mu_i>0$, $i = 1, \ldots, N$, that sum to 1 (an arbitrary normalization, but one that supports the idea of weighted average that we use later). From the point of view of optimization, each weight $\mu_i$ could be absorbed into $f_i$.  However, in Consensus Equilibrium we move beyond this optimization framework to the case in which the $f_i$ may be defined only implicitly or the case in which there is no optimization, but only mappings that play a role similar to the proximal maps that arise in the ADMM approach to solving \eqref{eq:split}.  The formulation in \eqref{eq:split} serves as motivation and the foundation on which we build.  

To extend the optimization framework of \eqref{eq:split} to consensus equilibrium, we start with $N$ vector-valued maps, $F_i: \RR^n \to \RR^n$, $i=1, \ldots, N$.  
The {\em Consensus Equilibrium} for these maps is defined as any solution $(x^*, \u^*) \in \RR^n \times \RR^{nN}$ that solves the equations
\begin{align}
F_i( x^* + u_i^* ) &= x^*, \ i=1, \ldots, N, \label{eq:CEF}\\
\bar{\u}_\mu^* &= 0.  \label{eq:CEu}
\end{align}
Here $\u$ is a vector in $\RR^{nN}$ obtained by stacking the vectors $u_1, \ldots, u_N$, and $\bar{\u}_\mu$ is the weighted average $\sum_{i=1}^N \mu_i u_i$.  

In order to relate consensus equilibrium to consensus optimization,
first consider the special case in which each $f_i:\Re^n\to  \Re \cup \{+\infty\}$ in \eqref{eq:split} is a proper, closed, convex function and each $F_i$ is a corresponding proximal map, i.e., a map of the form
\begin{equation}  \label{eq:prox}
F_i(x) = \argmin_v \left\{\frac{\|v-x\|^2}{2 \sigma^2} + f_i(v)\right\}. 
\end{equation}
Methods such as ADMM, Douglas-Rachford, and other variants of the proximal point algorithm apply these maps in sequence or in parallel with well-chosen arguments, together with some map to promote $x_i = z$ for all $i$, in order to solve \eqref{eq:split};  see e.g., \cite{bertsekas1996, boyd2011distributed, combettes2011proximal, ecksteinThesis, eckstein1992douglas, Lions1979, boyd-primer}.   
In the setting of Bayesian regularized inversion, each $f_i$ represents a data fidelity term or regularizing function.  We allow for the possibility that $f_i$ enforces some hard constraints by taking on the value $+\infty$.

Our first theorem states that when the maps $F_i$ are all proximal maps as described above, then the solutions to the consensus equilibrium problem are exactly the solutions to the consensus optimization problem of equation~\eqref{eq:split}.  In this sense, Consensus Equilibrium encompasses the optimization framework of \eqref{eq:split}. 

\medskip
\begin{theorem} \label{thm1} 
For $i=1, \ldots, N$, let $f_i$  be a proper, lower-semicontinuous, convex function on $\Re^n$, and let $\mu_i > 0$ with $\sum_{i=1}^N \mu_i = 1$.  Define $f = \sum_{i=1}^N \mu_i f_i$, and assume $f$ is finite on some open set in $\Re^n$.  Let $F_i$ be the proximal maps as in \eqref{eq:prox}.  Then the set of solutions to the consensus equilibrium equations of~\eqref{eq:CEF} and \eqref{eq:CEu}
is exactly the set of solutions to the minimization problem \eqref{eq:split}.
\end{theorem}
  
\medskip
The proof is contained in the appendix.

\subsection{Consensus Equilibrium Beyond Optimization}

Theorem~\ref{thm1} tells us that consensus equilibrium extends consensus optimization, but as noted above, the novelty of consensus equilibrium is not as a recharacterization of \eqref{eq:split} in the case of proximal maps but rather as a framework that applies even when some of the $F_i$ are not proximal mappings and there is {\em no underlying optimization problem to be solved}.  The Plug-and-Play reconstruction method of \cite{venkatakrishnan2013school}, which yields high quality solutions for important applications in tomography \cite{sreehari2016TCI} and denoising \cite{rond2015poisson}, is to our knowledge, the first method to use denoisers that do not arise from an optimization for regularized inversion.  As we show below, the CE framework also encompasses the Plug-and-Play framework in that if Plug-and-Play converges, then the result is also a CE solution.  However, Plug-and-Play grew out of ADMM, and the operators that yield convergence in ADMM are more limited than we would like.   Hence, for both consensus optimization and Plug-and-Play priors, CE encompasses the original method but also allows for a wider variety of operators and solution algorithms.   

An important point about moving beyond the optimization framework is that a given set of maps $F_i$ may lead to multiple possible CE solutions.  This may also happen in the optimization framework when the $f_i$ are not strictly convex since there may be multiple local minima.  In the optimization case, the objective function can sometimes be used to select among local minima. The analogous approach for consensus equilibrium is to choose a solution that minimizes the size of $\bu_\mu^*$, e.g. the $L_1$ or $L_2$ norm of $\bu^*$.  This corresponds in some sense to minimizing the tension among the competing forces balanced to find equilibrium.

\section{Solving the Equilibrium Equations}

In this section, we rewrite the CE equations as an unconstrained system of equations and then use this to express the solution in terms of a fixed point problem.  We also discuss particular methods of solution, including novel preconditioning methods and methods to solve for a wide range of possible $\F$.  We first introduce some additional notation.  For $\v \in \RR^{nN}$, with $\v = (v_1^T, \ldots, v_N^T)$ and each $v_j \in \RR^n$, define $\F, \G_\mu:\RR^{nN} \to \RR^{nN}$ by
\begin{equation} \label{eqn:FG}
\F(\v) = \begin{pmatrix} F_1(v_1) \\ \vdots \\ F_N(v_N) \end{pmatrix} \ \text{ and } \ 
 \G_\mu(\v) = \begin{pmatrix} \bv_\mu \\ \vdots \\ \bv_\mu \end{pmatrix} , 
 \end{equation}
where $\G_\mu$ has the important interpretation of redistributing the weighted average of the vector components given by $\bv_\mu = \sum_{i=1}^N \mu_i v_i$ across each of the output components.

Also, for $x \in \RR^n$, let $\hat x$ denote the vector obtained by stacking $N$ copies of $x$.  With this notation, the CE equations are given by
\begin{align}
\F(\hat x^* + \u^*) &= \hat x^*,  \label{eq:CEG}\\
\bu_\mu^* &= \0.  \nonumber
\end{align}

This notation allows us to reformulate the CE equations as the solution to a system of equations. 

\medskip
\begin{theorem}  \label{thm:T} 
The point $(x^*, \u^*)$ is a solution of the CE equations~\eqref{eq:CEF} and \eqref{eq:CEu}  if and only if the point $\v^* = \hat x^* + \u^*$ satisfies   $\bv_\mu^* = x^*$ and
\begin{align}  \label{eq:FG}
\F(\v^*) =\G_\mu(\v^*).
\end{align}
\end{theorem}
\medskip

\begin{proof}
Let $(x^*, \u^*)$ be a solution to the CE equations, and let $\v^* = \hat x^* + \u^*$.  Linearity of $\G_\mu$ together with $\bu_\mu^* = \0$ give $\G_\mu(\v^*) = \hat x^*$, so in particular, $\bv_\mu^* = x^*$.  Using this in \eqref{eq:CEG} gives \eqref{eq:FG}.

Conversely, if $\v^*$ satisfies \eqref{eq:FG}, define $x^* = \bv_\mu^*$ and $\u^* = \v^* - \hat x^*$.  Then \eqref{eq:CEF} and \eqref{eq:CEu} are satisfied by definition of $x^*$ and \eqref{eq:FG}. 
\end{proof}

We use this to reformulate consensus equilibrium as a fixed point problem.  

\medskip
\begin{corollary} \label{cor:T} (Consensus equilibrium as fixed point.)
The point $(x^*, \u^*)$ is a solution of the CE equations~\eqref{eq:CEF} and \eqref{eq:CEu}  if and only if the point $\v^* = \hat x^* + \u^*$ satisfies   $\bv_\mu^* = x^*$ and
\begin{align}  \label{eq:fixed}
(2\G_\mu - \I)(2\F - \I)\v^* = \v^*. 
\end{align}
\end{corollary}
\medskip

When $\F$ is a proximal map for a function $f$, then $2 \F - \I$ is known as the reflected resolvent of $f$.  Discussion and results concerning this operator can be found in \cite{Bauschke2011,eckstein1992douglas,boyd-precondition} among many other places.  This fixed point formulation is closely related to the fixed point formulation for minimizing the sum of two functions using Douglas-Rachford splitting; this is seen clearly in section 4 of \cite{Giselsson2017} among other places.  The form given here is somewhat different in that the reflected resolvents are computed in parallel and then averaged, as opposed to the standard sequential form.  Beyond that, the novelty here is in the equivalence of this formulation with the CE formulation.  

\begin{proof}[Proof of Corollary~\ref{cor:T}]
By Theorem~\ref{thm:T}, $(x^*, \u^*)$ is a solution of \eqref{eq:CEF} and \eqref{eq:CEu}  if and only if $\v^* = \hat x^*+ \u^*$ satisfies $\bv^* = x^*$ and \eqref{eq:FG}.  From \eqref{eq:FG} we have $(2\F - \I) \v^* = (2\G_\mu - \I) \v^*$.  A calculation shows that $\G_\mu \G_\mu = \G_\mu$, so $(2\G_\mu - \I)(2\G_\mu - \I) = \I$ by linearity of $\G_\mu$.  Hence applying $2\G_\mu-\I$ to both sides gives \eqref{eq:fixed}.  Reversing these steps returns from \eqref{eq:fixed} to \eqref{eq:FG}.  
\end{proof}

\subsection{Anisotropic Preconditioned Mann Iteration for Nonexpansive Maps}

Define $\T = (2\G_\mu-\I)(2\F-\I)$.  When $\T$ is nonexpansive and has a fixed point, we can use Mann iteration to find a fixed point of $\T$ as required by \eqref{eq:fixed}.  For a fixed parameter $\rho \in (0,1)$, this takes the form
\begin{equation}  \label{eq:MannFG}
\w^{k+1} = (1-\rho) \w^k + \rho \T(\w^k),  
\end{equation}
with iterates guaranteed to converge to a fixed point of $\T$. In the context of minimization problems in which $\F$ and $\G$ are both proximal maps, and depending on the choice of $\rho$, iterations of this form are essentially variants of the proximal point algorithm and give rise to the (generalized) Douglas-Rachford algorithm, the Peaceman-Rachford algorithm, and the ADMM algorithm, including over-relaxed and under-relaxed variants of ADMM.  In the case of $N=2$ and $\rho = 0.5$, the form in \eqref{eq:MannFG} is equivalent up to a change of variables to the standard ADMM algorithm; other values of $\rho$ give over-relaxed and under-relaxed variants.  Early work in this direction appears in \cite{ecksteinThesis} and \cite{eckstein1992douglas}.  A concise discussion is found in \cite{boyd-precondition}, which together with \cite{Giselsson2014} provides a preconditioned version of this algorithm in the case of $N=2$.  This preconditioning is obtained by replacing the original minimization of $f(x) + g(x)$ by minimization of $f(Dq) + g(Dq)$, which gives rise to iterations involving the conjugate proximal maps $D^{-1} F_D(D q)$, where $F_D$ is the proximal map for $f$ as in \eqref{eq:prox} using the norm $\| \cdot \|_{(D D^T)^{-1}}$ in place of the usual Euclidean metric. \cite{boyd-precondition} includes some results about the rate of convergence as a function of $D$.  In some cases, a larger value of $\rho$ leads to faster convergence relative to $\rho = 0.5$.  There are also results on convergence in the case that fixed $\rho$ is replaced by a sequence of $\rho_k$ such that $ \sum_k \rho_k (1-\rho_k) = \infty$ \cite{Bauschke2011}.  Further discussion and early work on this approach is found in \cite{ecksteinThesis,eckstein1992douglas}.  With some abuse of nomenclature, we use ADMM below to refer to Mann iteration with $\rho = 0.5$.

Here we describe an alternative preconditioning approach for the Mann iteration in which we use an invertible linear map $\H$ in place of the scalar $\rho$ in \eqref{eq:MannFG}.  In this approach, $\T$ can be any nonexpansive map and $\H$ can be any symmetric matrix with $\H$ and $\I-\H$ both positive definite.   

\begin{theorem} \label{thm:FilteredMann}
 Let $\H$ be a positive definite, symmetric matrix and let $\T$ be nonexpansive on $\RR^{nN}$ with at least one fixed point.  Suppose that the largest eigenvalue of $\H$ is strictly less than $1$.  For any $\v^0$ in $\RR^{nN}$, define
\begin{equation} \label{eq:PrecondMann}
\v^{k+1} = (\I - \H) \v^k + \H \T(\v^k)
\end{equation}
for each $k \geq 0$.  Then the sequence $\{\v^k\}$ converges to a fixed point of $\T$.  
\end{theorem}

\medskip

The idea of the proof is similar to the proof of convergence for Mann iteration given in \cite{boyd-primer}, but using a norm that weights differently the orthogonal components arising from the spectral decomposition of $\H$.  The proof is contained in the appendix.  

We note that in the case that each $f_i$ is a proper, closed, convex function on $\RR^n$ and $F_i$ is the proximal map as in \eqref{eq:prox}, then the map $2\F - \I$ is nonexpansive, so this preconditioning method can be used to find a solution to the problem in \eqref{eq:split}.  The asymptotic rate of convergence with this method is not significantly different from that obtained with the isotropic scaling obtained with a scalar $\rho$.  However, we have found this approach to be useful for accelerating convergence in certain tomography problems in which various frequency components converge at different rates, leading sometimes to visible oscillations in the reconstructions as a function of iteration number.  An appropriate choice of the preconditioner $H$ can dampen these oscillations and provide faster convergence in the initial few iterations. We will explore this example and related algorithmic considerations further in a future paper.  

\subsection{Beyond nonexpansive maps}  

The iterative algorithms obtained from \eqref{eq:MannFG} and \eqref{eq:PrecondMann} give guaranteed global convergence when $T$ is nonexpansive and $\rho$ (or $H$) satisfy appropriate conditions.  However, the iterates of \eqref{eq:MannFG} may still be convergent for more general maps $T$.  We illustrate this behavior in Case~1 of Section~\ref{sec:stochastic}.

In fact, when $T$ is differentiable at a fixed point, the rate of convergence is closely related to the spectrum of the linearization of $T$ near this fixed point.  The parameter $\rho$ in \eqref{eq:MannFG} maintains a fixed point at $\w^* = T(\w^*)$ but changes the linear part of the iterated map to have eigenvalues $\mu_j = \rho \lambda_j + (1-\rho)$, where $\lambda_1, \ldots, \lambda_n$ are the eigenvalues of the linear part of $T$.  The iterates of \eqref{eq:MannFG} converge locally exactly when all of these $\mu_j$ are strictly inside the unit disk in the complex plane.  This can be achieved for sufficiently small $\rho$ precisely when the real part of each $\lambda_j$ is  less than 1.  Since there is no constraint on the complex part of the eigenvalues, the map $T$ may be quite expansive in some directions.  In this case, the optimal rate of convergence is obtained when $\rho$ is chosen so that the eigenvalues $\mu_j$ all lie within a minimum radius disk about the origin.

The use of $\rho$ to affect convergence rate and/or to promote convergence is closely related to the ideas of  overrelaxation and underrelaxation as applied in a variety of contexts. See e.g. \cite{Hackbusch94} for further discussion in the context of linear systems. In the current setting, the use of $\rho < 1/2$ is a form of underrelaxation that is related to methods for iteratively solving ill-posed linear systems.   In the following theorem, the main idea is to make use of underrelaxation in order to shrink the eigenvalues of the resulting operator to the unit disk and thus guarantee convergence.

\begin{theorem}  \label{thm:spectrum}  (Local convergence of Mann iterates)
Let $\F: \Re^n\rightarrow \Re^n$ and $\G: \Re^n\rightarrow \Re^n$ be maps such that $\T = (2\G-\I)(2\F-\I)$ has a fixed point $\w^*$.  Suppose that $\T$ is differentiable at $\w^*$ and that the Jacobian of $\T$ at $\w^*$ has eigenvalues $\lambda_1, \ldots, \lambda_n$ with the real part of $\lambda_j$ strictly less than 1 for all $j$.  Then there is $\rho \in (0,1)$ and an open set $U$ containing $\w^*$ such that for any initial point $\w^0$ in $U,$ the iterates defined by \eqref{eq:MannFG} converge to $\w^*$.
\end{theorem}

The proof of this theorem is given in Appendix~\ref{AppProofs}.

\subsection{Newton's Method}

By formulating the CE as a solution to $\F(\v) - \G_\mu(\v) = 0$, we can apply a variety of root-finding methods to find solutions.  Likewise, rewriting \eqref{eq:fixed} as $T(\v) - \v = 0$ gives the same set of options.  

Let $H$ be a smooth map from $\RR^n$ to $\RR^n$.  The basic form of Newton's method for solving $H(x) = 0$ is to start with a vector $x_0$ and look for a vector $dx$ to solve $H(x_0 + dx) = 0$.  A first-order approximation gives $J_H(x_0) dx = -H(x_0)$, where $J_H(x_0)$ is the Jacobian of $H$ at $x_0$.  If this Jacobian is invertible, this equation can be solved for $dx$ to give $x_1 = x_0 + dx$ and the method iterated.  There are a wide variety of results concerning the convergence of this method with and without preconditioning, with various inexact steps, etc.  An overview and further references are available in \cite{Nocedal2006}.  

For large scale problems, calculating the Jacobian can be prohibitively expensive.  
The Jacobian-Free Newton-Krylov (JFNK) method is one approach to avoid the need for a full calculation of the Jacobian.  Let $J = J_H(x_0)$.  The key idea in Newton-Krylov methods is that instead of trying to solve $J dx = -H(x_0)$ exactly, we instead minimize $\|H(x_0) + J dx\|$ over the vectors $dx$ in a Krylov subspace, $K_j$.  This subspace is defined by first calculating the residual $r = -H(x_0)$ and then taking 
$$ K_j = \text{span}\{r, Jr, \ldots, J^{j-1} r\}. $$
The basis in this form is typically highly ill-conditioned, so the Generalized Minimal RESidual method (GMRES) is often used to produce an orthonormal basis and solve the minimization problem over this subspace.  This form requires only multiplication of a vector by the Jacobian, which can be approximated as 
$$ Jr \approx \frac{H(x_0 + \epsilon r) - H(x_0)}{\epsilon}. $$
Applying this to produce $K_j$ requires $j$ applications of the map $H$ together with the creation of the Arnoldi basis elements, which can then be used to find the minimizing $dx$ by reducing to a standard least squares problem of dimension $j$. Various stopping conditions can be used to determine an appropriate $j$.   These calculations take the place of the solution of $J dx = -H(x_0)$.  In cases for which there are many contracting directions and only a few expanding directions for Newton's method near the solution point, the JFNK method can be quite efficient.  
A more complete description with a discussion of the benefits and pitfalls, approaches to preconditioning, and many further references is contained in \cite{Knoll2004}.

We note in connection with the previous section that if $\H$ is chosen to be $(\I - J_\T(\v^k))^{-1}$ in \eqref{eq:PrecondMann}, then the choice of $\v^{k+1}$ in Theorem~\ref{thm:FilteredMann} is an exact Newton step applied to $\I - \T$.  That is, the formula for the step in Newton's method in this case is
$$ \v^{k+1}-\v^k = -(\I - J_T(v^k))^{-1} (\I - \T) \v^k, $$
or
$$ \v^{k+1} = (\I - \H) \v^k + \H \T \v^k,$$
which is the same as the formula in Theorem~\ref{thm:FilteredMann}.  

In the examples below, we use standard Newton's method applied to both $\F-\G$ and $\T - \I$ in the first example and JFNK applied to $\F-\G$ in the second.  Because of the connection with Mann iteration just given, we use the term Newton Mann to describe Newton's method applied to $\T-\I$.

\subsection{Other approaches}  \label{sec:Expand}

An alternative approach is to convert the CE equations back into an optimization framework by considering the residual error norm given by
\begin{align} 
R(\v) \stackrel{\Delta}{=}  \left \| \F(\v) - \G_\mu(\v) \right \|
\label{Residual-Norm}
\end{align}  
and minimizing $R^2(\v)$ over $\v$.  
Assuming that a solution of the CE equations exists, then that solution is also a minimum of this objective function.  In the case that $\F$ is twice continuously differentiable, a calculation using the facts that $R(\v^*) = 0$ and that $\G_\mu$ is linear shows that the Hessian of $R^2(\v)$ is $2 A^T A + O(\|\v - \v^*\|)$, where $A\v = J_\F(\v^*)\v - \G_\mu\v$.    Hence $R^2(\v)$ is locally convex near $\v^*$ as long as $A$ has no eigenvalue equal to 0.  Since $\G_\mu$ is a projection, its only eigenvalues are 0 and 1, hence this is equivalent to saying that $J_\F(\v^*)$ does not have 1 as an eigenvalue.  If $A$ does have an eigenvalue 0, then a perturbation of the form $\F_\epsilon(\v) = \F(\v) + \epsilon \v$ produces a unique solution, which can be followed in a homotopy scheme as $\epsilon$ decreases to 0.

One possible algorithm for this approach is the Gauss-Newton method, which can be used to minimize a sum of squared function values and which does not require second derivatives.

We note that the residual error of equation~(\ref{Residual-Norm}) is also useful as a general measure of convergence when computing the CE solution; we use this in plots below.

Other candidate solution algorithms include the forward-backward algorithm and related algorithms as presented in \cite{combettes2011proximal}.  
We leave further investigation of efficient algorithms for solving the CE equations to future research.

\section{Experimental Results}

Here we provide some computational examples of varying complexity. 
For each of these examples, at least one of the component maps $F_i$ is not a proximal mapping,
so the traditional optimization formulation of equations~(\ref{eq:RegularizedInverse}) or~(\ref{eq:split}) is not applicable.

We start with a toy model in 2 dimensions to illustrate the ideas, then follow with some more complex examples.

\subsection{Toy model}
In this example we have $v_1 = (v_{11}, v_{12})^T$, $v_2 = (v_{21}, v_{22})^T$, both in $\RR^2$, and maps $F_1$ and $F_2$ defined by
\begin{align*}
F_1(v_1) &= \left(I + \sigma^2 {A^T A} \right)^{-1}(v_1 + \sigma^2 A^T y) \\
F_2(v_2) &=   1.1(v_{21} + 0.2, v_{22}-0.2 \sin(2 v_{22}))^T.
\end{align*}
In this case, $F_1$ is a proximal map as in \eqref{eq:prox} corresponding to $f(x) = \|Ax-y\|^2/2$ and $F_2$ is a weakly expanding map designed to illustrate the properties of consensus equilibrium.  
We use $\sigma = 1$ and 
$$
A = \left[ \begin{array}{cc} 0.3&0.6\\ 0.4&0.5 \end{array} \right], \ \ y = \left[ \begin{array}{cc} 1\\ 1 \end{array} \right]. 
$$
We take $\mu_1 = \mu_2 = 0.5$ and so write $G$ for $G_\mu$.  We apply Newton iterations  to $\F(\v) - \G(\v) = 0$ and to the fixed point formulation $T(\v) - \v = 0$.  In both cases, the Jacobian of $F_2$ is evaluated only at the initial point.  

\begin{figure} \label{fig:fields} 
\begin{center}
\begin{minipage}{0.65 \textwidth}
\includegraphics[width= 0.99\textwidth]{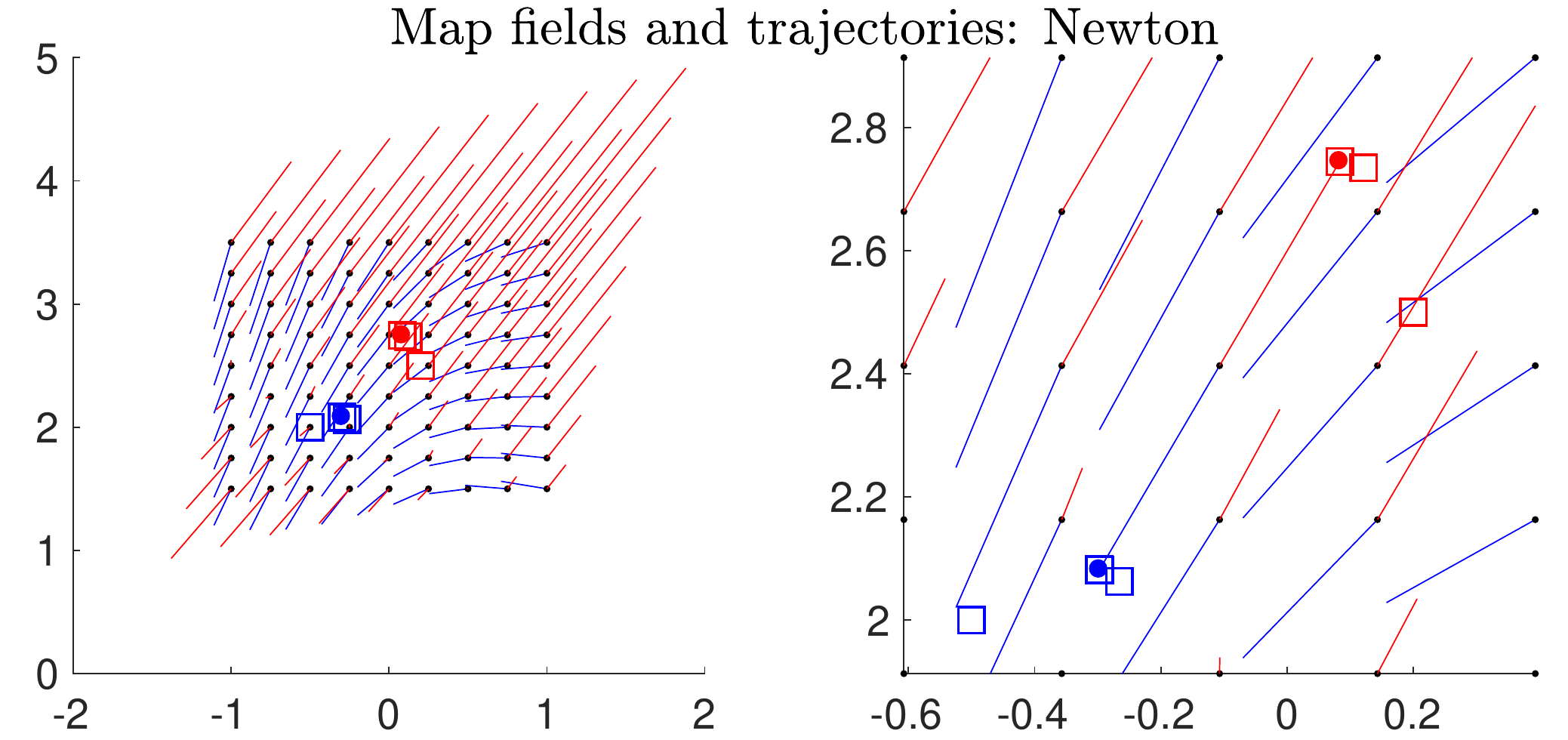}
\end{minipage}
\hfill
\begin{minipage}{0.3 \textwidth}
\includegraphics[width=0.95 \textwidth]{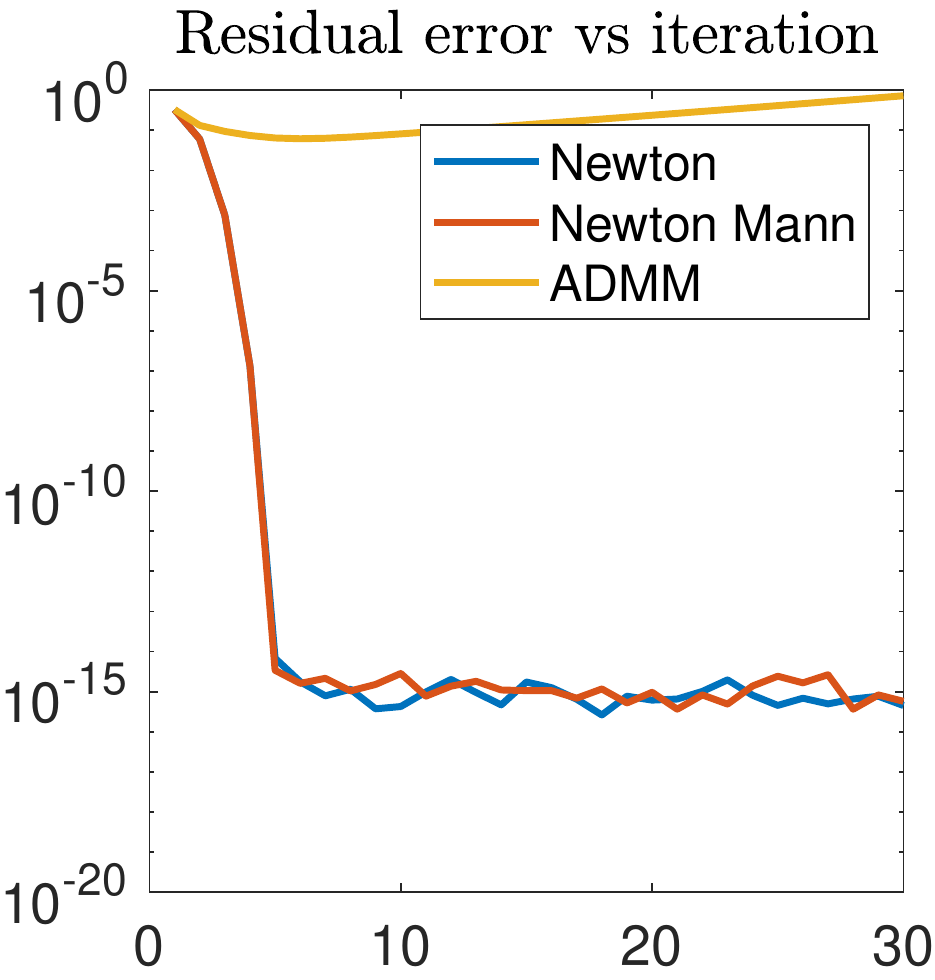}
\end{minipage}
\end{center}
\caption{Left:  Map fields and trajectories for 2-dimensional toy example using Newton's method applied to solve $\F(\v) - \G(\v) = 0$.   Blue segments show the map $v_1 \mapsto F_1(v_1)$, red segments show $v_2 \mapsto F_2(v_2)$, with black dots showing the common endpoints of these maps.  Blue and red open squares show the points $v_1^k$ and $v_2^k$, respectively.  Filled red and blue circles show the CE solution.  Middle:  Zoom in near the fixed point of the plot on the left.   Right:  Error $\|\F(\v^k) - \G(\v^k)\|$ as a function of iteration for Newton's method applied to $\F - \G$, Newton's method applied to $\T - \I$ (labeled as Newton Mann), and standard Mann iteration with $\rho = 0.5$ (labeled as ADMM).  }
\end{figure}

Figure~\ref{fig:fields} shows the vectors obtained from each of the maps $F_1$ and $F_2$.  
Blue line segments are vectors from a point $v_1$ to $F_1(v_1)$, and red line segments are vectors from a point $v_2$ to $F_2(v_2)$.  
The starting points of each pair of red and blue vectors are chosen so that they have a common ending point, signified by black dot.  
Open squares show the trajectories of $v_1^k$ in blue and $v_2^k$ in red.  
The trajectories converge to points for which the corresponding red and blue vectors have a common end point and are equal in magnitude and opposite in direction;  this is the consensus equilibrium solution.  The plots shown are for Newton's method  applied to $\F-\G$; the plots for Newton's method applied to $\T - \I$ are similar (not shown).  In the right panel of this figure, we use the true fixed point to plot error versus iterate for this example using all 3 methods.   The expansion in $F_2$ prevents ADMM from converging in this example.

\subsection{Stochastic matrix}  \label{sec:stochastic}
The next example uses the proximal map form for $F_1$ as in the previous example, although now with dimension $n=100$.  
$A$ and $y$ were chosen using the random number generator rand in Matlab, approximating the uniform distribution on $[0,1]$ in each component.  
The map $F_2$ has the form $F_2(v) = Wv$;  here $W$ is constructed by first choosing entries at random in the interval $[0,1]$ as for $A$, then replacing the diagonal entry by the maximum entry in that row (in which case the maximum entry may appear twice in one row), and then normalizing so that each row sums to 1.  
This mimics some of the features in a weight matrix appearing in denoisers such as non-local means \cite{buades2005review} but is designed to allow us to compute an exact analytic solution of the CE equations.  
In particular, since $W$ is not symmetric, $F_2$ cannot be a proximal map, as shown in \cite{sreehari2016TCI}.  

In order to illustrate possible convergence behaviors, 
we first fix the matrices $A$ and $W$ and the vector $y$ as above
and then use a one-parameter family of maps $F_{2,r}(v) = rWv + (1-r) I/2$.  When $0 \leq r \leq 1$, this map averages $W$ and $I/2$.  The map $I/2$ is a proximal map as in \eqref{eq:prox} with $\sigma = 1$ and $f_i(v) = \|v\|^2/2$;  i.e., the proximal map associated with a quadratic regularization term.  In the framework of Corollary~\ref{cor:T}, the map $F_{2,r}$ satisfies $2F_{2,r}(v)-v = r(2W - I)v$.  Hence the scaling of $r$ controls the expansiveness of one of the component maps in $2\F-\I$, and hence the expansiveness of the operator in \eqref{eq:fixed} through the averaging operator $G_\mu$.    
For the examples here, we choose $r$ to be $1.02$ and $1.06$.  As described below, with appropriate choices of parameters, the Jacobian-free Newton Krylov method converges for both examples, while ADMM converges for the first one only.  

Recall that if the Lipschitz constant, $L(T)$, is strictly less than $1$, then the operator $T$ is a contraction,
and if $L(T)\leq 1$ we say it is nonexpansive.
Moreover, for linear operators, $L(T)=\sigma_{max}$ where $\sigma_{max}$ is the maximum singular value of $T$;
and $\sigma_{max} \geq | \lambda_{max} |$ where $\lambda_{max}$ is the eigenvalue with greatest magnitude.

\begin{figure}[t] \label{fig:trajectories2} 
\begin{center}
\includegraphics[width=0.44 \textwidth]{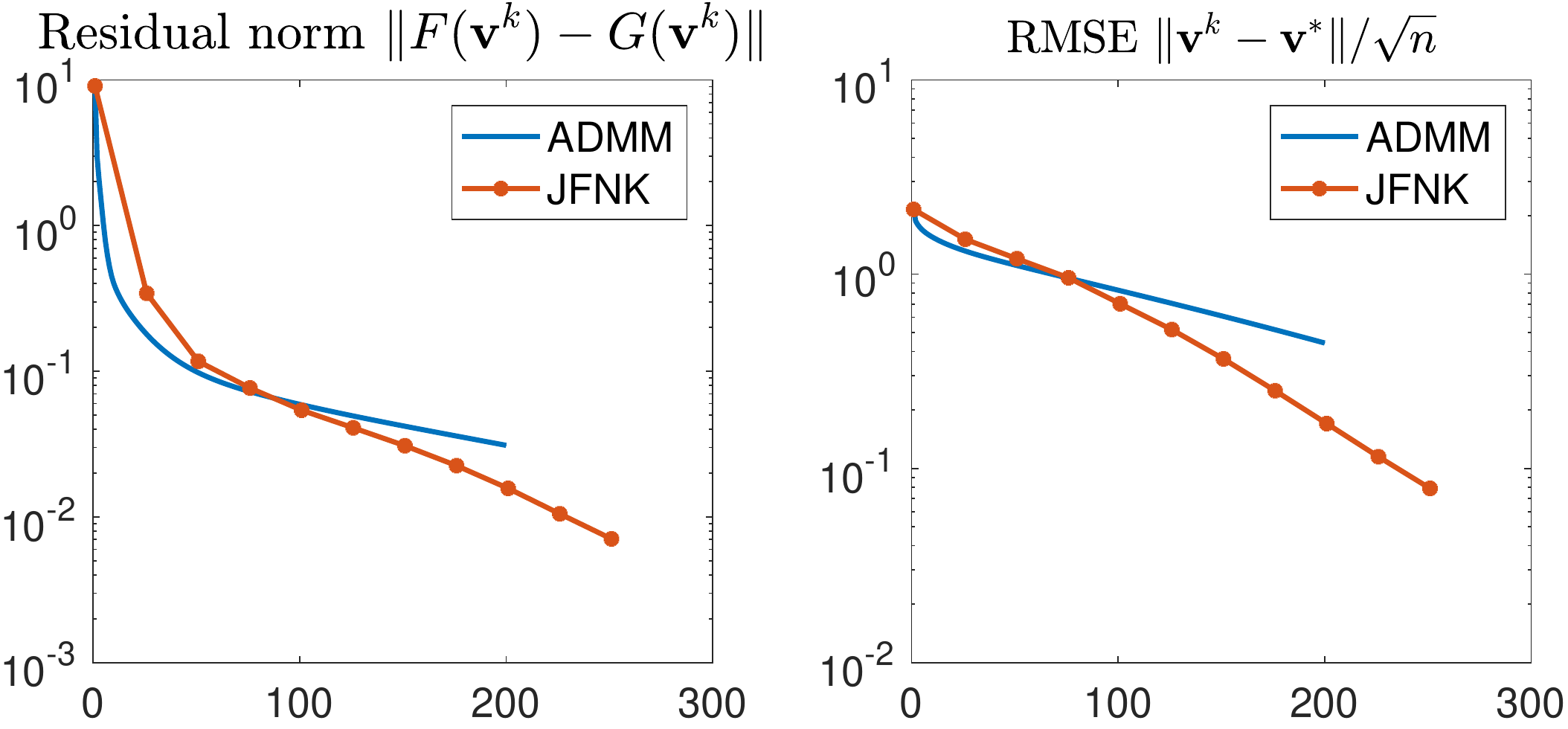}
\mbox{\hspace{12pt}}
\includegraphics[width=0.44 \textwidth]{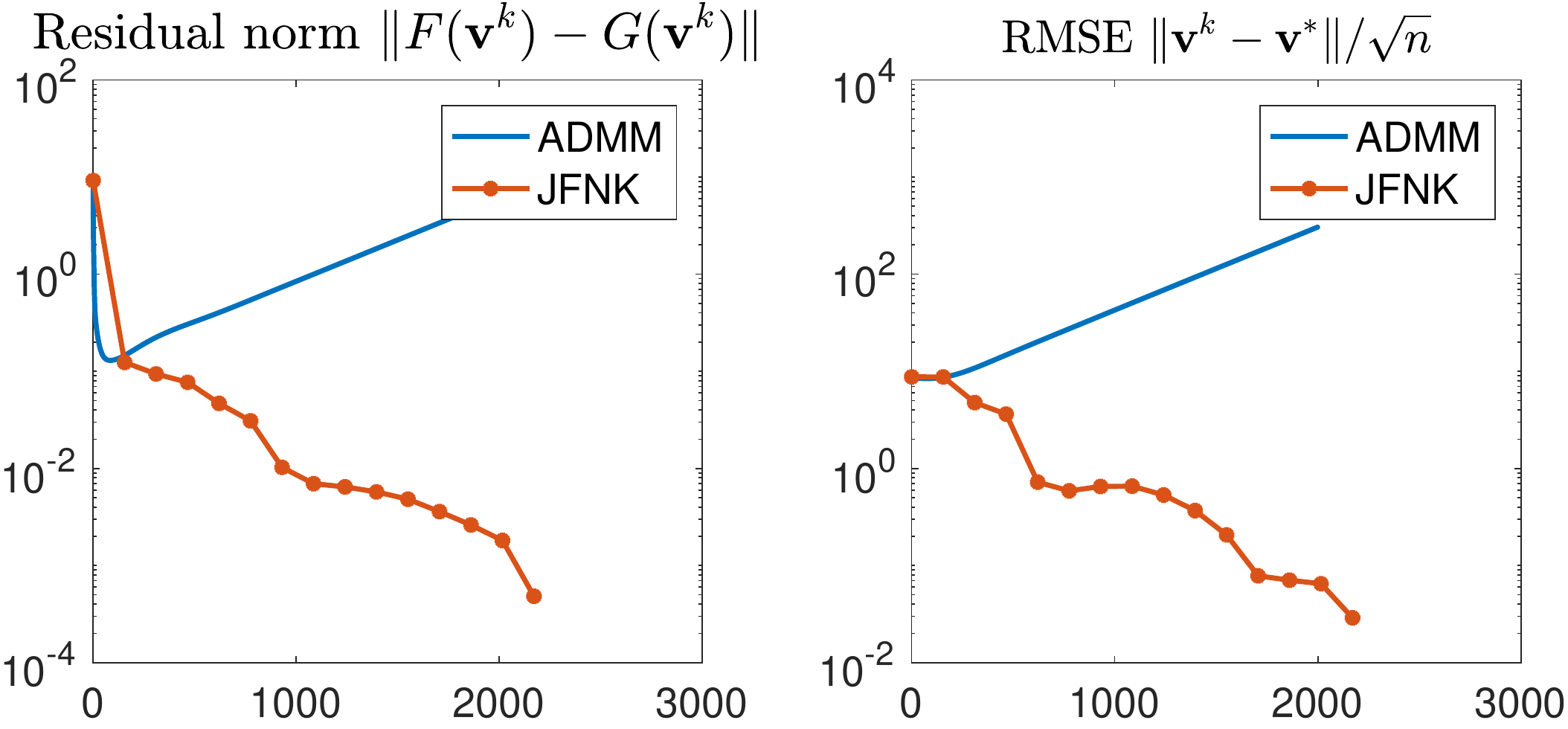}
\end{center}
\caption{Residual norm $\|F(\v^k) - \G(\v^k)\|$ and root mean square error $\|v^k - v^*\|/\sqrt{n}$ versus iteration.
Left 2 panels: (Case~1) $T$ has Lipschitz constant larger than 1 but all eigenvalues have real part strictly less than 1.  Both methods converge.
Right 2 panels:  (Case~2) $T$ has an eigenvalue with real part larger than 1.  JFNK converges, while ADMM diverges.}
\end{figure}

\bigskip

\begin{itemize}%[leftmargin=1in]
\item[Case 1,]  {\em $r=1.02$}:
In this case, $\T$ has Lipschitz constant $L(\T) > 1$, and the conditions of Theorem~\ref{thm:FilteredMann} or similar theorems on the convergence of Mann iteration for the convergence of nonexpansive maps do not hold.
However, in this case, $\T$ is affine (linear map plus constant) and all eigenvalues of the linear part of $(\T + \I)/2$ lie strictly inside the unit circle.   From \eqref{eq:MannFG} with $\rho = 1/2$ and basic linear algebra, this means that Mann iteration converges.  This is confirmed in Figure~\ref{fig:trajectories2}.  In this example, convergence for Mann iteration can be improved by taking $\rho$ to be $0.8$, in which case the convergence is marginally better than that for JFNK.  For this example, we used a Krylov subspace of dimension 10, so that each Newton step requires 10 function evaluations.  This is indicated by closed circles in the plots.  
\medskip

\item[Case 2,]  {\em $r=1.06$}:
In this case, $\T$ has Lipschitz constant $L(\T) > 1$, and there is an eigenvalue with real part approximately $1.0039$, so averaging $\T$ with the identity as in Mann iteration will maintain an eigenvalue larger than one.  In particular, Mann iteration with $\rho = 1/2$ (labelled as ADMM) does not converge, but the JFNK algorithm does.  For this example, we used a Krylov subspace of dimension 75, so that each Newton step requires 75 function evaluations.  This is indicated by closed circles in the plots.
\end{itemize}
\medskip

\subsection{Image Denoising with Multiple Neural Networks}
The third example we give is an image denoising problem using multiple deep neural networks. This problem is more complex in that we use several neural networks, none of which is tuned to match the noise in the image to be denoised.  Nevertheless, we show that CE is often able to outperform each individual network.
The images and code for this section are available at \cite{CEcode}.

The forward model of image denoising is described by the following linear equation:
\begin{equation*}
y = x + \eta,
\end{equation*}
where $x \in \RR^n$ is latent unknown image, $\eta \sim \mathcal{N}(0,\sigma_\eta^2I)$ is i.i.d. Gaussian noise, and $y \in \RR^n$ is the corrupted observation. Our motivation is to find an estimate $x^* \in \RR^n$ by solving the consensus equilibrium equation analogous to the classical maximum-a-posteriori approach:
\begin{equation}
x^*
= \underset{x\in \RR^n}{\mbox{argmin}} \;\; \frac{1}{2\sigma_\eta^2}\|y - x\|^2 - \log p(x),
\end{equation}
where $p(x)$ is the prior of $x$.  However, instead of a prior function, which would induce a proximal map, we will use a set of convolutional neural networks, which will play the role of regularization in the way that a prior term does, but which are almost certainly not themselves proximal maps for any function.  

To define the CE operators $F_i$, we consider a set of $K$ image denoisers. Specifically, we use the denoising convolutional neural network (DnCNN) proposed by Zhang et al. \cite{Zhang2017_tip}. In the code provided by the authors \footnote{Code available at https://github.com/cszn/ircnn}, there are five DnCNNs trained at five different noise levels: $\sigma_1 = 10/255$, $\sigma_2 = 15/255$, $\sigma_3 = 25/255$, $\sigma_4 = 35/255$ and $\sigma_5 = 50/255$. In other words, the user has to choose the appropriate DnCNN to match the actual noise level $\sigma_\eta$. In the CE framework, we see that $F_i$ is the operator
\begin{equation}
F_i(v_i) = \mbox{DnCNN}(v_i, \; \mbox{with denoising strength} \; \sigma_{i}).
\end{equation}
The $(K+1)^{\mathrm{st}}$ CE operator $F_{K+1}$ is the proximal map of the likelihood function:
\begin{equation} \label{eq:FK1}
F_{K+1}(v_{K+1}) = \underset{x\in \RR^n}{\mbox{argmin}} \;\; \frac{1}{2\sigma_\eta^2}\|y-x\|^2 + \frac{1}{2\sigma^2}\|v_{K+1}-x\|^2,
\end{equation}
where $\sigma$ is an internal parameter controlling the strength of the regularization $\|v_{K+1}-x\|^2$. In this example, we set $\sigma = \sigma_\eta$ for simplicity.

\begin{figure}[ht]
\begin{tabular}{cccc}
& & & \\
\includegraphics[width=0.22\linewidth]{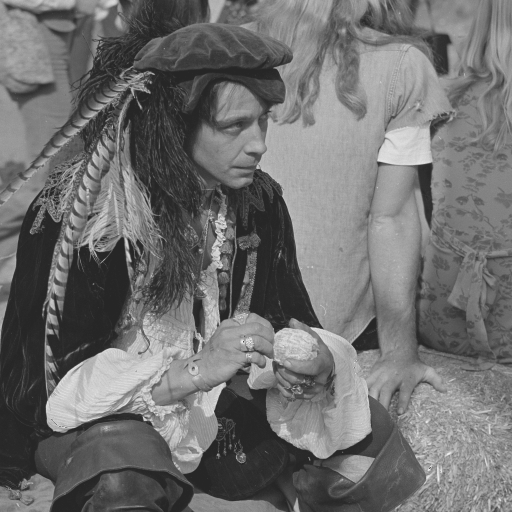}&
\includegraphics[width=0.22\linewidth]{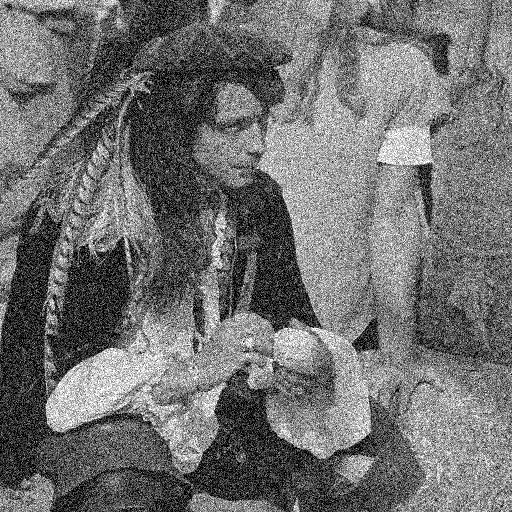}&
\includegraphics[width=0.22\linewidth]{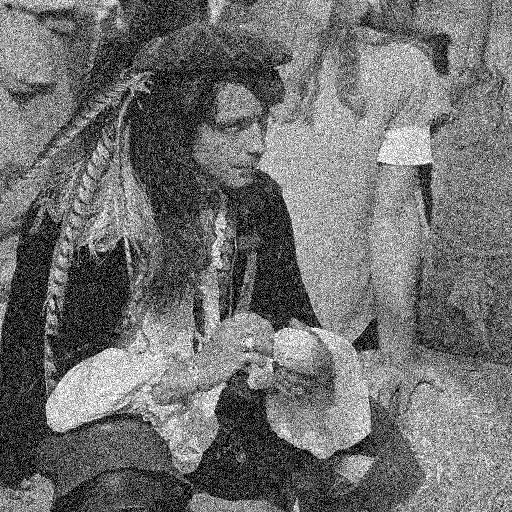}&
\includegraphics[width=0.22\linewidth]{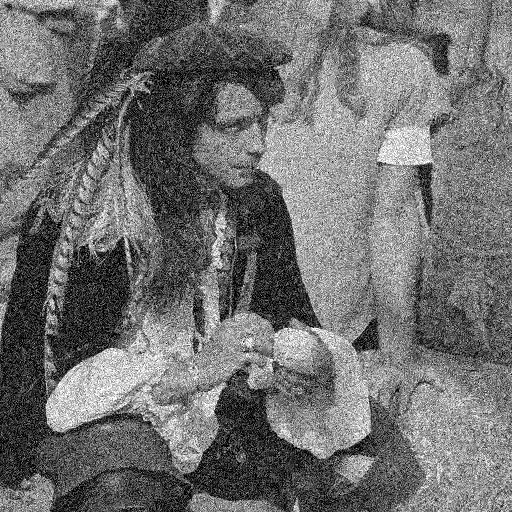}\\
Noiseless & Noisy $\sigma_\eta = 40/255$ & $\mbox{DnCNN}_{10}$, 16.67dB & $\mbox{DnCNN}_{15}$, 17.53dB\\
\includegraphics[width=0.22\linewidth]{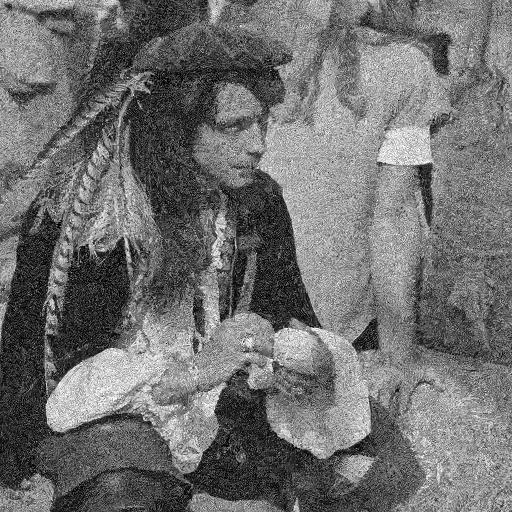}&
\includegraphics[width=0.22\linewidth]{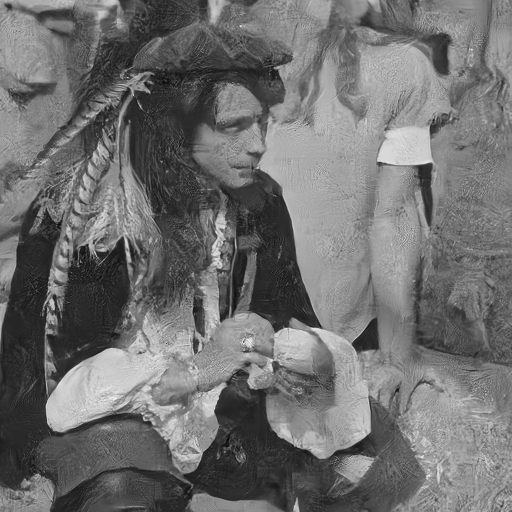}&
\includegraphics[width=0.22\linewidth]{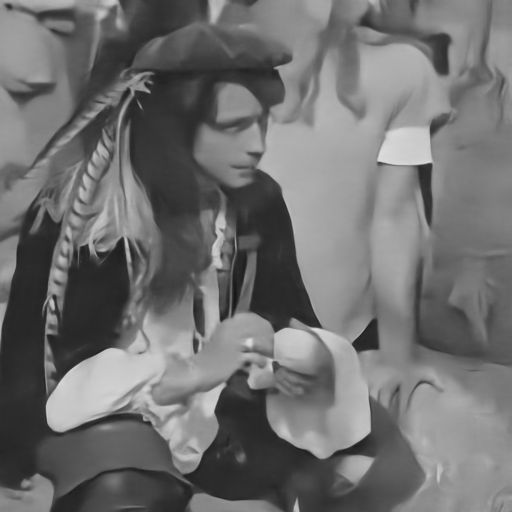}&
\includegraphics[width=0.22\linewidth]{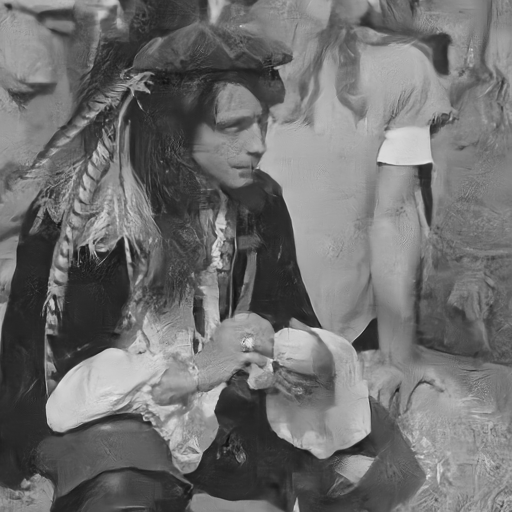}\\
 $\mbox{DnCNN}_{25}$, 19.92dB & $\mbox{DnCNN}_{35}$, 26.44dB & $\mbox{DnCNN}_{50}$, 27.39dB & CE, 27.77dB\\
 \end{tabular}
\caption{Image denoising experiment for Man512 when $\sigma_\eta = 40/255$.
Notice that the consensus equilibrium (CE) result has the highest SNR when compared to individual convolutional neural network denoisers train on varying noise levels.}
\label{fig:denoising results1}
\end{figure}

To make the algorithm more adaptive to the data, we use weights $\mu_i = \frac{p_i}{ \sum_{i=1}^{K+1} p_i}$, where
\begin{equation}  \label{eq:weight}
p_i = \exp\left\{-\frac{(\sigma_\eta - \sigma_i)^2}{2h^2}\right\},  \quad \mbox{and} \quad  p_{K+1} = \sum_{i=1}^K p_i.
\end{equation}
In this pair of equations, $p_i$ measures the deviation between the actual noise level $\sigma_\eta$ and the denoising strength of the neural networks $\sigma_i$. The parameter $h = 5/255$ controls the cut off. Therefore, among the five neural networks, $p_i$ weights more heavily the relevant networks. The $(K+1)^{\mathrm{st}}$ weight $p_{K+1}$ is the weight of the map to fit to data. Its value is chosen to be the sum of the weights of the denoisers to provide appropriate balance between the likelihood and the denoisers.

\begin{figure}[ht]
\centering
\begin{tabular}{cccc}
\includegraphics[width=0.22\linewidth]{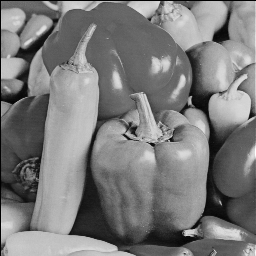}&
\includegraphics[width=0.22\linewidth]{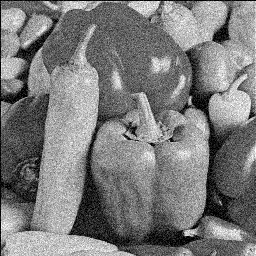}&
\includegraphics[width=0.22\linewidth]{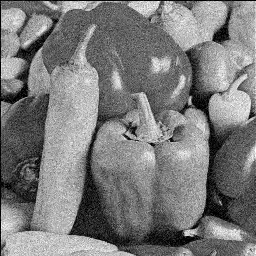}&
\includegraphics[width=0.22\linewidth]{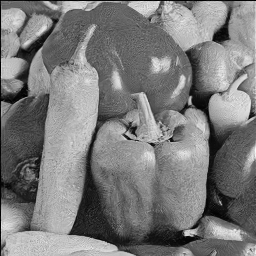}\\
Noiseless & Noisy $\sigma_\eta = 20/255$ & $\mbox{DnCNN}_{10}$, 23.98dB & $\mbox{DnCNN}_{15}$, 28.25dB\\
\includegraphics[width=0.22\linewidth]{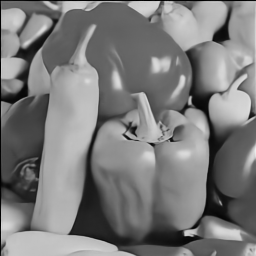}&
\includegraphics[width=0.22\linewidth]{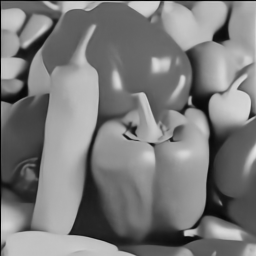}&
\includegraphics[width=0.22\linewidth]{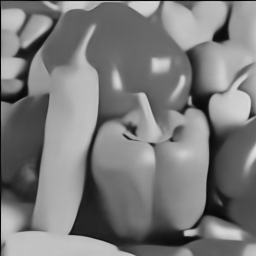}&
\includegraphics[width=0.22\linewidth]{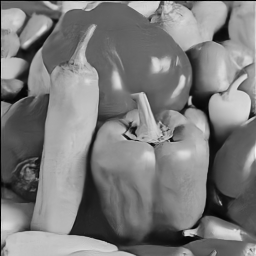}\\
$\mbox{DnCNN}_{25}$, 31.33dB & $\mbox{DnCNN}_{35}$, 29.51dB & $\mbox{DnCNN}_{50}$, 27.93dB & CE, 31.79dB\\
\end{tabular}
\caption{Image denoising experiment for Peppers256 when $\sigma_\eta = 20/255$. 
Notice that the consensus equilibrium (CE) result has the highest SNR when compared to individual convolutional neural network denoisers train on varying noise levels.}
\label{fig:denoising results2}
\end{figure}

\begin{table}[t]
\centering
\caption{Image denoising results for actual noise level $\sigma_\eta \in \{20,30,40\}/255$.}
\begin{tabular}{c|ccccc|cc|c}
      & \multicolumn{5}{|c|}{DnCNN} &  & & Matched \\
Image & 10 & 15 & 25 & 35 & 50 & Baseline & CE & DnCNN\\
\hline
$\sigma = 20/255$ & &         &           &           &            &      &                 & \\
Barbara512 & 23.99 &  	28.02 &  	30.49 &  	28.11 &  	25.71  & 29.80 & \textbf{30.97} & 31.02\\
Boat512    & 23.98 &  	27.92 &  	30.61 &  	28.73 &  	27.03 & 29.86 & \textbf{31.08} & 31.15\\
Cameraman256&24.12 &  	28.04 &  	30.20 &  	28.52 &  	27.20 & 29.88 & \textbf{31.05} & 31.07\\
Hill512    & 23.93 &  	27.81 &  	30.34 &  	28.68 &  	27.20 & 29.78 & \textbf{30.88} & 30.92\\
House256   & 24.03 &  	28.70 &  	33.70 &  	32.32 &  	30.69 & 31.38 & \textbf{33.82} & 33.97\\
Lena512    & 24.07 &  	28.59 &  	33.06 &  	31.13 &  	29.59 & 31.12 & \textbf{33.35} & 33.47\\
Man512     & 23.94 &  	27.89 &  	30.41 &  	28.46 &  	27.02 & 29.79 & \textbf{31.00} & 31.08\\
Peppers256 & 23.98 &  	28.25 &  	31.33 &  	29.51 &  	27.93 & 30.27 & \textbf{31.79} & 31.80\\
\hline
$\sigma = 30/255$ &  &          &           &           &            &   &  &      \\
Barbara512 & 19.49 & 21.01 & 26.86 & 28.49 & 26.15 & 28.49 & \textbf{28.98} & 28.86\\
Boat512    & 19.48 & 21.02 & 26.96 & 28.92 & 27.20 & 28.81 & \textbf{29.38} & 29.36\\
Cameraman256&19.62 & 21.14 & 27.11 & 28.62 & 27.23 & 28.77 & \textbf{29.18} & 29.25\\
Hill512    & 19.48 & 21.01 & 26.78 & 28.91 & 27.38 & 28.79 & \textbf{29.35} & 29.33\\
House256   & 19.44 & 21.05 & 28.48 & 32.17 & 30.68 & 31.24 & \textbf{32.39} & 32.32\\
Lena512    & 19.49 & 21.10 & 28.18 & 31.37 & 29.78 & 30.71 & \textbf{31.73} & 31.69\\
Man512     & 19.48 & 21.01 & 26.87 & 28.77 & 27.21 & 28.73 & \textbf{29.28} & 29.29\\
Peppers256 & 19.48 & 21.00 & 27.40 & 29.45 & 27.96 & 29.25 & \textbf{29.85} & 29.81\\
\hline
$\sigma = 40/255$ &  &          &           &           &            &  &      &\\
Barbara512 & 16.69& 	17.54& 	19.93& 	26.05& 	26.51 & 26.57 & \textbf{27.14} & 27.32\\
Boat512    & 16.66& 	17.52& 	19.93& 	26.50& 	27.36 & 27.02 & \textbf{27.82} & 28.12\\
Cameraman256&16.81& 	17.65& 	20.13& 	26.51& 	27.24 & 26.98 & \textbf{27.68} & 27.96\\
Hill512    & 16.66& 	17.53& 	19.92& 	26.47& 	27.56 & 27.05 & \textbf{27.90} & 28.23\\
House256   & 16.61& 	17.50& 	20.06& 	28.30& 	\textbf{30.57} & 29.00 & 30.47 & 31.04\\
Lena512    & 16.66& 	17.55& 	20.04& 	28.01& 	29.88 & 28.67 & \textbf{29.95} & 30.38\\
Man512     & 16.67& 	17.53& 	19.92& 	26.44&	27.39 & 26.99 & \textbf{27.77} & 28.11\\
Peppers256 & 16.67& 	17.53& 	19.93& 	26.79& 	27.87 & 27.29 & \textbf{28.09} & 28.38\\
\hline
\end{tabular}
\label{table:denoising results}
\end{table}

Figures~\ref{fig:denoising results1} and~\ref{fig:denoising results2} show some results using noise levels of $\sigma_\eta = 20/255$ and $40/255$, respectively.
Notice that none of these noise levels is covered by the trained DnCNNs.
Table~\ref{table:denoising results} shows resulting SNR values for the full set of experiments using 8 test images and 3 noise levels of $\sigma_\eta = 20/255, 30/255, 40/255$.  The results in the center of the table indicate the result of applying an individual CNN to the noisy image.  Because of the form of $F_{K+1}$ in \eqref{eq:FK1} and $\sigma = \sigma_\eta$, the result of this single application of the CNN is the same as the CE solution obtained by using only that single CNN together with $F_{K+1}$.  

Notice that in almost all cases the consensus equilibrium of the full group has the highest PSNR when compared to the individual application of the DnCNNs.
Also, the improvement in terms of the PSNR is quite substantial for noise levels $\sigma_\eta = 20/255$ and $\sigma_\eta = 30/255$. 
For $\sigma_\eta = 40/255$, CE still offers PSNR improvement except for House256, which is an image with many smooth regions.
In addition, visual inspection of the images shows that the CE result yields the best visual detail while also removing the noise most effectively.
While DnCNN denoisers can be very effective, they must be trained in advance using the correct noise level. 
This demonstrates that the consensus equilibrium can be used to generate a better result by blending together multiple pre-trained DnCNNs.

In order to illustrate that the CE solution outperforms a well-chosen linear combination of the outputs from each denoiser, we report a baseline combination result in Table~\ref{table:denoising results}. The baseline results are generated by
$$
\xhat_{\mathrm{baseline}} = \sum_{i=1}^n \mu_i \xhat_i,
$$
where $\{\xhat_i\}$ are the initial estimates provided by the denoisers, and $\mu_i$ is defined through \ref{eq:weight} without $p_{K+1}$.  That is, we use the same weights as those for CE, excluding the weight for the likelihood term and rescaled to sum to 1 after this exclusion. Therefore, $\xhat_{\mathrm{baseline}}$ can be considered as a linear combination of the initial estimates, with weights defined by the distance between the current noise level and the trained noise levels. The results in Table~\ref{table:denoising results} show that while $\xhat_{\mathrm{baseline}}$ very occasionally outperforms the best of the individual denoisers, it is usually worse than the best individual denoiser and is uniformly worse than CE. In the last column of Table~\ref{table:denoising results} we show the result of DnCNN trained at a noise level matched with the actual noise level. It is interesting to note that CE compares favorably with the matched DnCNN in many cases, except for large sigma where the matched DnCNN is uniformly better.

We note that \cite{YQWang2014} uses a linear transformation depending on the noise level of a noisy image in order to match the noise level of a trained neural network, and then applies the inverse linear transformation to the output.  This provides another approach to the example above but doesn't include the ability of CE to combine multiple sources of influence without a predetermined conversion from one to the other. We should also point out the recent work of Choi et al. \cite{Choi_Elgendy_Chan_2018} which demonstrates an optimal mechanism of combining image denoisers.

\section{Conclusion}
We presented a new framework for image reconstruction, which we term Consensus Equilibrium (CE).
The distinguishing feature of the CE solution is that it is defined by a balance among operators rather than the minimum of a cost function.
The CE solution is given by the consensus vector that arises simultaneously from the balance of multiple operators, which may include various kinds of image processing operations.
In the case of conventional regularized inversion, for which the optimization framework holds, the CE solution agrees with the usual MAP estimate, but CE also applies to a wide array of problems for which there is no corresponding optimization formulation.

We discussed several algorithms for solving the CE equations, including a novel anisotropic preconditioned Mann iteration and a Jacobian-free Newton Krylov method.
We also introduced a novel precondition method for accelerating the Mann iterations used to solve the CE equations.
There is a great deal of room to explore other methods for finding CE solutions as well as for formulating other equilibrium conditions.  

Our experimental results, on a variety of problems with varying complexity,
demonstrate that the Consensus Equilibrium approach can solve problems for which there is no corresponding regularized optimization and can in some cases achieve consensus results that are better than any of the individual operators.
In particular, we showed how the Consensus Equilibrium can be used to integrate a number of CNN denoisers, thereby achieving a better result than any individual denoiser.

\section*{Acknowledgments}
We thank the referees for many helpful remarks that have improved this paper substantially.

\begin{appendices}

\section{Appendix: Proofs} \label{AppProofs}

\begin{proof}[Proof of Theorem~\ref{thm1}] 
In order to use $\sigma^2 >0$ as in \eqref{eq:prox}, we multiply the objective function in \eqref{eq:split} by $\sigma^2$, which does not change the solution.  Define the Lagrangian associated with this scaled problem as 
$$ L(x, (x_i)_{i=1}^N, (\lambda_i)_{i=1}^N) = \sum_{i=1}^N(\sigma^2 \mu_i f_i(x_i) + (x-x_i)^T \lambda_i), $$
where the $\lambda_i \in \RR^n$ are the Lagrange multipliers for the equality constraints $x_i = x$.  Since the $f_i$ are all convex and lower-semicontinuous, the first order KKT conditions are necessary and sufficient for optimality \cite[Theorem 28.3]{Rockafellar1997}.  At a solution point $(x^*, (x_i^*)_{i=1}^N, (\lambda_i^*)_{i=1}^N)$, these conditions are given by 
\begin{align*}
\nabla_x L(x^*, (x_i^*)_{i=1}^N, (\lambda_i^*)_{i=1}^N) &= 0\\
\partial_{x_i} L(x^*, (x_i^*)_{i=1}^N, (\lambda_i^*)_{i=1}^N) &\ni 0, \forall i=1,\ldots, N\\
x_i^* - x^* &= 0, \forall i=1,\ldots, N,  
\end{align*}
where $\partial_{x_i}$ is the subdifferential with respect to $x_i$.  These convert to 
\begin{align}
\sum_{i=1}^N \lambda_i^* &= 0  \label{eq:lambdai}\\  
\sigma^2 \mu_i \partial f_i(x_i^*) &\ni \lambda_i^*, \ \forall i=1,\ldots, N \label{eq:partfi}\\
x_i^* &= x^*, \ \forall i=1,\ldots, N.  \label{eq:xix}
\end{align}
Define $u_i^* = \lambda_i^* / \mu_i $, in which case \eqref{eq:lambdai} is the same as \eqref{eq:CEu}.  Next, use $x_i^* = x^*$ from \eqref{eq:xix} in \eqref{eq:partfi} and cancel $\mu_i$ to get $\sigma^2 \partial f_i(x^*) \ni u_i^*$ for all $i$.  
Adding $x^*$ to both sides gives $ x^* + \sigma^2 \partial f_i(x^*) \ni x^* + u_i^*,$
or
$$ (I + \sigma^2 \partial f_i)(x^*) \ni x^* + u_i^*.$$
Since the $f_i$ are convex and $\sigma^2 > 0$, we can invert to get $x ^*= (I + \sigma^2 \partial f_i)^{-1}(x^* + u_i^*)$.   From \cite[Proposition~16.34]{Bauschke2011}, this is equivalent to \eqref{eq:CEF} in the case that $F_i$ is the proximal map of \eqref{eq:prox}.
\end{proof}

\begin{proof}[Proof of Theorem~\ref{thm:FilteredMann}]
Since $H$ is symmetric and positive definite, there is an orthogonal matrix $Q$ and a diagonal matrix $\Lambda$ with $\Lambda_{jj} = \lambda_j > 0$ for all $j$ and $H = Q \Lambda Q^T$.  Let $q_j$ be the $j$th column of $Q$, and let $\pi_j v = (q_j^T v) q_j$ be orthogonal projection onto the span of $q_j$.  Define the associated norm $\|v\|_j = \|\pi_j v\|$.   
 Also, let $\lambda$ be the product $\lambda_1 \cdots \lambda_N$, and let $\hat \lambda_j = \lambda/\lambda_j$ (i.e., the product of all $\lambda_1$ through $\lambda_N$ except $\lambda_j$).  Define the weighted norm
$$ \|v\|_{H^{-1}}^2 = v^T H^{-1} v = \sum_j \lambda^{-1}_j \|v\|_j^2, $$
which is equivalent to the standard norm on $\RR^N$.  

By assumption, there is a fixed point $v^* = Tv^*$.  Using $\pi_j H = \lambda_j \pi_j$ and applying $\pi_j$ to both sides of the definition of $v^{k+1}$ gives
\begin{align*}
\|  v^{k+1} - v^* \|_j^2  = \|(1 &- \lambda_j) \pi_j v^k +  \lambda_j \pi_j Tv^k \\ 
&- (1- \lambda_j) \pi_j v^* -  \lambda_j \pi_j T v^*\|^2.
\end{align*}
Here and below, we use $v^* = Tv^*$ freely as needed.  
As in \cite{boyd-primer}, we use the equality $\|(1-\theta) a + \theta b\|^2 = (1 - \theta) \|a\|^2 + \theta \|b\|^2 - \theta(1-\theta) \|a-b\|^2$, which holds for $\theta$ between 0 and 1 and can be verified by expanding both sides as a function of $\theta$.  In our case, we have $\theta =  \lambda_j \in (0,1)$ from the assumptions on $H$.  After conversion back to the norm $\| \cdot \|_j$, this yields
\begin{align*}
\|  v^{k+1} - v^* \|_j^2 = (1 &-  \lambda_j) \|v^k - v^*\|_j^2 +  \lambda_j \|T v^k - Tv^*\|_j^2 \\
&-  \lambda_j (1 -  \lambda_j) \|v^k - Tv^k\|_j^2.  
\end{align*}
Summing with weights $\lambda_j^{-1}$ gives
\begin{align*}
\sum_j \lambda_j^{-1} \|v^{k+1} - v^*\|^2_j = \sum_j (\lambda_j^{-1} &- 1) \|v^k - v^*\|_j^2 
+ \sum_j  \|T v^k - T v^*\|_j^2 \\
&-  \sum_j \lambda_j^{-1} \lambda_j (1 -  \lambda_j) \|v^k - Tv^k\|_j^2.
\end{align*}
Since $T$ is nonexpansive, the right hand side is bounded above by replacing $Tv^k-Tv^*$ with $v^k-v^*$ in the second sum.  This new sum then exactly cancels the term arising from $-1$ in the first sum.  Let $c$ be the minimum over $j$ of $\lambda_j (1 -  \lambda_j)$, and note that $c>0$ since $\lambda_j < 1$ for each $j$ by assumption.  Putting these together and re-expressing in the $H^{-1}$ norm gives
$$ \|v^{k+1} - v^*\|_{H^{-1}}^2 \leq \|v^k - v^*\|_{H^{-1}}^2 -  c \|v^k - T v^k\|_{H^{-1}}^2. $$
The remainder of the proof is nearly identical to that in \cite{boyd-primer}; we include it for completeness.  Iterating in the first term on the right hand side, we obtain
\begin{equation} \label{eqn:vk1}
 \|v^{k+1} - v^*\|_{H^{-1}}^2 \leq  \|v^1 - v^*\|_{H^{-1}}^2 -  c \sum_{i=1}^k \|v^j - T v^j\|_{H^{-1}}^2, 
 \end{equation}
and hence
$$ \sum_{i=1}^k \|v^j - T v^j\|_{H^{-1}}^2 \leq \frac{1}{ c} \|v^1 - v^*\|_{H^{-1}}^2.  $$
In particular, $\|v^j - Tv^j\|_{H^{-1}}$ and hence $\|v^j - Tv^j\|$ tend to 0 as $j$ tends to $\infty$.  This also implies that 
$$ \min_{j=1, \ldots, k} \|v^j - Tv^j\|_{H^{-1}}^2 \leq \frac{1}{ c k} \|v^1 - v^*\|_{H^{-1}}^2. $$

Finally, note that \eqref{eqn:vk1} implies that the sequence $\{v^k\}$ is bounded, hence has a limit point, say $\hat v$.  Since $I-T$ is continuous and $v^k - Tv^k$ converges to 0, we have $\hat v = T \hat v$.  Using $\hat v$ in place of $v^*$ in \eqref{eqn:vk1}, we see that $\|v^k - \hat v\|_{H^{-1}}$ decreases monotonically to 0, hence $v^k$ converges to $\hat v$.  
\end{proof}

\begin{proof}[Proof of Theorem~\ref{thm:spectrum}]

Let $\T_\rho$ denote the map $(1-\rho) \I + \rho \T$.  
Let $\mu_j(\rho) = (1-\rho) + \rho \lambda_j$ and note that $\T_\rho$ has eigenvalues $\mu_1, \ldots, \mu_n$.  Since the real part of $\lambda_j$ is less than 1, the line segment defined by $\mu_j(\rho)$ for $\rho$ in the interval $[0,1]$ has a nonempty intersection with the open unit disk in the complex plane.  For each $j$, there is some $\eps_j > 0$ so that this intersection contains the set  of points $\mu_j(\rho)$ for $\rho$ in $(0, \eps_j]$.  Taking $\eps_0$ to be the minimum of the $\eps_j$ and taking $\rho$ in the interval $(0, \eps_0]$, there exists $r<1$ for which $|\mu_j(\rho)| \leq r < 1$ for all $j$.

For this choice of $\rho$, let $A$ be the Jacobian of $\T_\rho$ at the fixed point, $\w^*$, which we may assume is the origin. The Schur triangulation gives a unitary matrix $Q$ and an upper triangular matrix $U$ with $U = Q^{-1} A Q$.  Write $U = \Lambda + U'$ with $\Lambda$ diagonal and $U'$ zero on the diagonal.  
Let $u_\text{max}$ be the maximum of $|U_{i,j}'|$ over all entries in $U'$.    For $\eps>0$, define $D$ to be the diagonal $n \times n$ matrix with $D_{i,i} = \eps^i$.  A computation shows that $D^{-1} U D$ has the same diagonal entries as $U$ but that each off-diagonal has the form $U_{i,j} \eps^{j-i}$ with $j>i$, hence is bounded by $\eps u_\text{max}$ in norm.  This plus the differentiability of $\T_\rho$ implies that for $x$ in a neighborhood of $0$,
\begin{align*}
\| D^{-1} Q^{-1} \T_\rho Q D x \| &= \| \Lambda x + D^{-1} U' D x\| + o(\|x\|) \\
& \leq (r + n \eps   u_\text{max}+ R(\|x\|))\|x\| ,  
\end{align*}
where $R(\|x\|)$ decreases to 0 as $\|x\|$ tends to 0.  Choosing $\eps$ and $\|x\|$ sufficiently small, we have $r + n \eps u_\text{max}+ R(\|x\|) < \beta$ for some $\beta < 1$. In this case we can iterate to obtain 
$$ \| D^{-1} Q^{-1} \T_\rho^k Q D x \| \leq \beta^k \|x\|.  $$
In other words, for $x^0$ in a neighborhood $N$ of the origin, the iterates $x^k = D^{-1} Q^{-1} \T_\rho^k Q D x^0$ converge geometrically to the origin.  Multiplying by $QD$ and labeling $\w^k = QDx^k$, we have $\w^k = \T_\rho^k \w^0$ converges geometrically to 0 for all $\w^0$ in the neighborhood $QDN$ of the origin.  
\end{proof}

\end{appendices}

{\small
\bibliographystyle{Formatting/siamplain}
\bibliography{References}
}

\end{document}

%% file: SIIMS_shared.tex
\usepackage{times}
\usepackage{epsfig}
\usepackage{graphicx}
\usepackage{amsmath}
\usepackage{amssymb}
\ifpdf
  \DeclareGraphicsExtensions{.eps,.pdf,.png,.jpg}
\else
  \DeclareGraphicsExtensions{.eps}
\fi

\newcommand{\TheTitle}{Plug-and-Play Unplugged:  \\Optimization Free Reconstruction using \\Consensus Equilibrium
} 
\newcommand{\TheAuthors}{G.T. Buzzard, S. Chan, S. Sreehari, C.A. Bouman}
\newcommand{\ShortTitle}{Consensus Equilibrium}

\headers{\ShortTitle}{\TheAuthors}

\title{{\TheTitle}\thanks{Submitted to the editors March 23, 2017.  Revised December 8, 2017, April 15, 2018, June 12, 2018.
\funding{G.T.B. was partially supported by NSF grants NSF DMS-1318894 and NSF CCF-1763896.  S.H.C. was partially supported by NSF grants NSF CCF-1718007 and NSF CCF-1763896. C.A.B was partially supported by NSF CCF-1763896.  S.S and C.A.B. were partially supported by an AFOSR/MURI grant  \#FA9550-12-1-0458, by UES Inc. under the Broad Spectrum Engineered Materials contract, and by the Electronic Imaging component of the ICMD program 
of the Materials and Manufacturing Directorate of the Air Force Research Laboratory, Andrew Rosenberger, program manager.}}}

\author{
  Gregery T. Buzzard\thanks{Department of Mathematics, Purdue University, West Lafayette, IN, USA
    (\email{buzzard@purdue.edu})}
  \and
  Stanley H. Chan\thanks{School of Electrical and Computer Engineering and Department of Statistics, Purdue University, West Lafayette, IN, USA
  (\email{stanchan@purdue.edu})}
  \and
  Suhas Sreehari\thanks{School of Electrical and Computer Engineering, Purdue University, West Lafayette, IN, USA (\email{ssreehar@purdue.edu})}
  \and
  Charles A. Bouman\thanks{School of Electrical and Computer Engineering and Weldon School of Biomedical Engineering, Purdue University, West Lafayette, IN, USA
   (\email{bouman@purdue.edu})}
}

\usepackage{amsopn}